\documentclass[sigconf]{acmart}
\usepackage{algorithm}
\usepackage{algorithmic}
\usepackage{color}
\usepackage{ulem}
\pdfoutput=1
\AtBeginDocument{%
  \providecommand\BibTeX{{%
    \normalfont B\kern-0.5em{\scshape i\kern-0.25em b}\kern-0.8em\TeX}}}

\copyrightyear{2020}
\acmYear{2020}
\setcopyright{acmcopyright}\acmConference[MM '20]{Proceedings of the 28th ACM International Conference on Multimedia}{October 12--16, 2020}{Seattle, WA, USA} 
\acmBooktitle{Proceedings of the 28th ACM International Conference on Multimedia (MM '20), October 12--16, 2020, Seattle, WA, USA}
\acmPrice{15.00}
\acmDOI{10.1145/3394171.3414048}
\acmISBN{978-1-4503-7988-5/20/10}

\begin{document}
\fancyhead{}
\title{Identity-Aware Attribute Recognition via Real-Time Distributed Inference in Mobile Edge Clouds}


\author{Zichuan Xu}
\affiliation{
  \institution{The Key Laboratory for Ubiquitous Network and Service Software of Liaoning Province, School of Software, Dalian University of Technology }
  \city{Dalian}
  \state{China}
  \postcode{116620}
}
\email{z.xu@dlut.edu.cn}

\author{Jiangkai Wu}
\affiliation{
  \institution{School of Software, Dalian University of Technology }
  \city{Dalian}
  \state{China}
  \postcode{116620}
}
\email{wujiangkai98@gmail.com}

\author{Qiufen Xia}
\authornote{Corresponding author.  This work is partially supported by the National Natural Science Foundation of China (Grant No. 61802048 and 61802047), the fundamental research funds for the central universities in China (Grant No. DUT19RC(4)035), DUT-RU Co-Research Center of Advanced ICT for Active Life, and the ``Xinghai Scholar Program'' in Dalian University of Technology, China. }
\affiliation{%
  \institution{The Key Laboratory for Ubiquitous Network and Service Software of Liaoning Province, International School of Information Science and Engineering, Dalian University of Technology, Dalian, China}
  \postcode{116620}
}
\email{qiufenxia@dlut.edu.cn}

\author{Pan Zhou}
\affiliation{%
  \institution{The Hubei Engineering Research Center on Big Data Security,
School of Cyber Science and Engineering, Huazhong University of Science and Technology, Wuhan, China}
  \postcode{430074}
}
\email{panzhou@hust.edu.cn}

\author{Jiankang Ren}
\affiliation{%
  \institution{School of Computer Science, Dalian University of Technology}
  \city{Dalian}
  \state{China}
  \postcode{116024}
}
\email{rjk@dlut.edu.cn}

\author{Huizhi Liang}
\affiliation{%
  \institution{Department of Computer Science, University of Reading}
  \streetaddress{Polly Vacher Building, Whiteknights}
  \city{Reading}
  \state{UK}
  \postcode{RG6 6AY}
}
\email{huizhi.liang@reading.ac.uk}



\begin{abstract}
With the development of deep learning technologies, attribute recognition and person re-identification (re-ID) have attracted extensive attention and achieved continuous improvement via executing computing-intensive deep neural networks in cloud datacenters. However, the datacenter deployment cannot meet the real-time requirement of attribute recognition and person re-ID, due to the prohibitive delay of backhaul networks and large data transmissions from cameras to datacenters. A feasible solution thus is to employ mobile edge clouds (MEC) within the proximity of cameras and enable distributed inference. 

In this paper, we design novel models for pedestrian attribute recognition with re-ID in an MEC-enabled camera monitoring system. We also investigate the problem of distributed inference in the MEC-enabled camera network. To this end, we first propose a novel inference framework with a set of distributed modules, by jointly considering the attribute recognition and person re-ID. 
We then devise a learning-based algorithm for the distributions of the modules of the proposed distributed inference framework, considering the dynamic MEC-enabled camera network with uncertainties. We finally evaluate the performance of the proposed algorithm by both simulations with real datasets and system implementation in a real testbed. Evaluation results show that the performance of the proposed algorithm with distributed inference framework is promising, by reaching the accuracies of attribute recognition and  person  identification  up to 92.9\% and  96.6\% respectively, and significantly reducing the inference delay by at least 40.6\%  compared with existing methods.
\end{abstract}





\maketitle

\section{Introduction}
Person re-identification (re-ID) and attribute recognition both are employed in critical applications of surveillance. Person re-ID is to identify pedestrians according to views from different video surveillance or spans a different time using a single camera~\cite{AJM15}, the main concern is whether the two pedestrian images taken by different cameras are the same person. While attribute recognition is to distinguish the presence of a set of attributes from an image, such as hairstyle, wearing, carrying, gender, age, etc. Person re-ID and attribute recognition have a common target to learn pedestrian features, and both rely on deep learning techniques. 
So combining the two techniques together makes sense for establishing an intelligent multi-camera  monitoring system. Accuracy and response time are crucial factors in intelligent multi-camera monitoring systems. 

Deep learning models are complex and resource-consuming, the resource-limited devices on the side of users are hard to compute their inference results. 
These deep learning services thus are mostly deployed in cloud datacenters with abundant GPU resources. However, all cameras continuously transmit their surveillance data to datacenters for processing. 
Such centralized architecture based on remote datacenters cannot satisfy the delay requirements of real-time analytics based on deep learning models, 
due to the congested links and severe transmission latency in cloud backhaul networks. 
With the development of 5G technique, mobile edge computing (MEC) consisting of multiple computing nodes 
emerges as an attractive and promising paradigm to host computation tasks as close as possible to the data sources and end users. In an MEC-enabled camera network, the inference tasks of person re-ID and attribute recognition can thus be offloaded to their nearby computing nodes, 
e.g., cloudlets or base stations with artificial intelligence accelerators, thereby reducing transmission delay significantly. 

Several works have been proposed to study the problem of person re-ID and attribute recognition~\cite{li2015multi, liu2018localization, li2018pose, sarafianos2018deep, liu2017hydraplus, zhao2018grouping, zhao2019recurrent, liu2018sequence,DBLP:journals/corr/abs-1805-02336,xie2019progressive,liu2019an,xiang2018homocentric,yang2018local,wang2018deep,DBLP:journals/corr/abs-1903-06325}. 
However, these existing algorithms cannot be simply combined to build a real-time system and cannot be directly employed in MEC-enabled camera networks. The reasons are as follows. 
First, conventional re-ID algorithm is used for offline image retrieval, so it needs to be redesigned for real-time identification. 
Existing attribute recognition algorithms typically ignore the identity-association of geographically distributed cameras. 
 So the re-ID and attribute recognition need to be jointly performed in a distributed manner to realize cross-camera attribute recognition.
Second, the distributed inference framework consisting of multiple modules has to be effectively deployed into the MEC-enabled camera network; otherwise, it may incur prohibitive inter-module communication delays or processing delays. 
Therefore, a finer grained integration of distributed inference framework and task assignment policy in MECs are urgently needed to accelerate the smart surveillance system. 

Realizing real-time attribute recognition and person re-ID in an MEC-enabled camera network presents major challenges. 
First, existing methods did not elaborate cross-camera attribute recognition, 
so how to identify the new coming pedestrians and fuse the recognized attributes in the level of identity are challenging. 
Second, conventional methods usually process person re-ID and attribute recognition tasks separately. However, the two tasks have a common target of feature learning and benefit each other. Designing a Convolutional Neural Network (CNN) model that can accomplish the two tasks with high accuracy is challenge. Third, inference tasks for person re-ID and attribute recognition are computationally intensive, incurring large computation workload to implement. 
It will lead to prohibitive computation latency if executing the inference only on surveillance cameras, due to the constrained computing capability of the cameras.  
Also, surveillance videos contain redundant data, which causes bandwidth waste and tremendous transmission latency if sending all videos to edge servers to process. 
Therefore, how to design an inference framework that can fully make use of the edge servers to minimize processing and transmission latencies is challenging. 
Fourth, the MEC-enabled camera network has many network uncertainties, such as the processing and transmission latencies of cloudlets and base stations of the network. 
How to distribute the dynamic inference framework to computing nodes while incorporating network uncertainties is another challenge. 

In this paper, we focus on designing the distributed inference framework for joint person re-ID and attribute recognition, and explore a non-trivial interplay between the inference framework and network uncertainties of an MEC-enabled camera network, with an aim to maximize the accuracy of the cross-camera attribute recognition and enable real-time inference. 

To the best of our knowledge, we are the first to consider cross-camera attribute recognition and person re-ID in an MEC-enabled camera network. 
The major contributions of this paper include  
\begin{itemize}
\item We formulate a problem of joint attribute recognition and person re-ID in an MEC-enabled camera network.
\item We design a distributed inference framework with a set of carefully designed modules that enables module distribution in a network. 
\item We propose a CNN-based model that trains attribute recognition and re-ID simultaneously. The accuracy is greatly improved while the size of the model is reduced by 55.2\%. 
\item We investigate the problem of efficiently and effectively distributing its modules to computing nodes in the MEC-enabled camera network, by considering various uncertainties of the network. 
\item We evaluate the performance of the proposed framework using real datasets in a real test-bed. 
Results demonstrate that the accuracies of attribute recognition and re-ID are up to 92.9\% and 96.6\% respectively, while the latency of distributed inference is reduced by at least 40.6\%. 
\end{itemize}


\section{Related work}\label{sec02} 
Upon the success of deep learning on automatic feature extraction, CNN-based methods are dominating the pedestrian attribute recognition~\cite{li2015multi, liu2018localization, li2018pose, sarafianos2018deep, liu2017hydraplus, zhao2018grouping, zhao2019recurrent, liu2018sequence} and re-ID~\cite{DBLP:journals/corr/abs-1805-02336,xie2019progressive,liu2019an,xiang2018homocentric,yang2018local,wang2018deep,DBLP:journals/corr/abs-1903-06325}. 
Person re-ID and attribute recognition tasks can share the same target of feature learning. Several deep learning approaches are thus proposed, which can execute both tasks in a single model~\cite{su2018multi, schumann2017person, lin2019improving}. 
For example, 
Lin et al.~\cite{lin2019improving} manually annotated attribute labels for two large-scale re-ID datasets: Market-1501 and DukeMTMC-reID, to learn the feature representation and attribute classifiers at the same time. 
Almost none of the listed references focused on the design of distributed inference in MEC environments. 

Some studies focused on proposing distributed CNN models. 
For example, Kang et al.~\cite{kang2017neurosurgeon} developed a regression model to predict the processing latency and energy consumption for each layer with the configuration information. 
Study in~\cite{jeong2018ionn} proposed a heuristic that divided a DNN into several partitions and distributed them to the edge servers incrementally. 
Teerapittayanon et al.~\cite{teerapittayanon2017distributed} mapped partitions of a DNN onto distributed end devices, edge servers and the cloud, by jointly training to consider the trade-off of transmission latency and model accuracy. 
However, the distribution of partitioned parts into different servers and the network uncertainties are not considered~\cite{kang2017neurosurgeon}, or the design and training of DNN model are neglected~\cite{jeong2018ionn}, or only general DNN models are adopted~\cite{teerapittayanon2017distributed}. 
Edge computing emerges as a new deployment paradigm of low-latency services for IoT or mobile applications. In the bottom layer, many general algorithms were proposed to optimize the Qos from the perspective of NFV~\cite{10.1145/3387705,xu2020nfv-enabled,8502709,xu2019nfv-enabled,Ren2020Efficient} or caching~\cite{Xu2020Learningforexception,Xu2020CacheOrNot,Xu2020CollaborateOrSeparate}. In the application layer, some researchers focused on video processing on edges~\cite{elgamal2018droplet:,wang2018bandwidth-efficient,8057318,elgamal2020serdab,hung2018videoedge:}, which are more similar to our scenario. However, none of the listed video processing references consider the design of a novel NN architecture or ignored the latency uncertainties. 

\section{Preliminary}\label{sec03}
In this section, we first introduce the system model, notations and notions. We then define the problems precisely.

\subsection{An MEC-enabled surveillance system}
We consider a surveillance environment with multiple cameras interconnected by an MEC-enabled network, as illustrated in Figure~\ref{topology}. 
Let $G = (C\cup V; E)$ be the MEC-enabled camera network, with a set $C$ of surveillance cameras, a set $V$ of edge servers that can implement the inference of deep learning model, and a set $E$ of links (wired or wireless) connecting cameras and edge servers. 
Denote by $c_i \in C$ a camera. 
Transferring video stream data out of each camera $c_i$ through links in $E$ consumes network bandwidth resources and incurs communication latencies.  
The latencies of network links depend on many factors, such as the congestion level of the link. Such latencies thus usually are uncertain and cannot be obtained in advance.
The computing resource in each edge server $v\in V$ is used to perform real-time person re-ID and attribute recognition, 
thereby incurring processing delay. 

\begin{figure}[htbp]
\centerline{\includegraphics[scale=0.22]{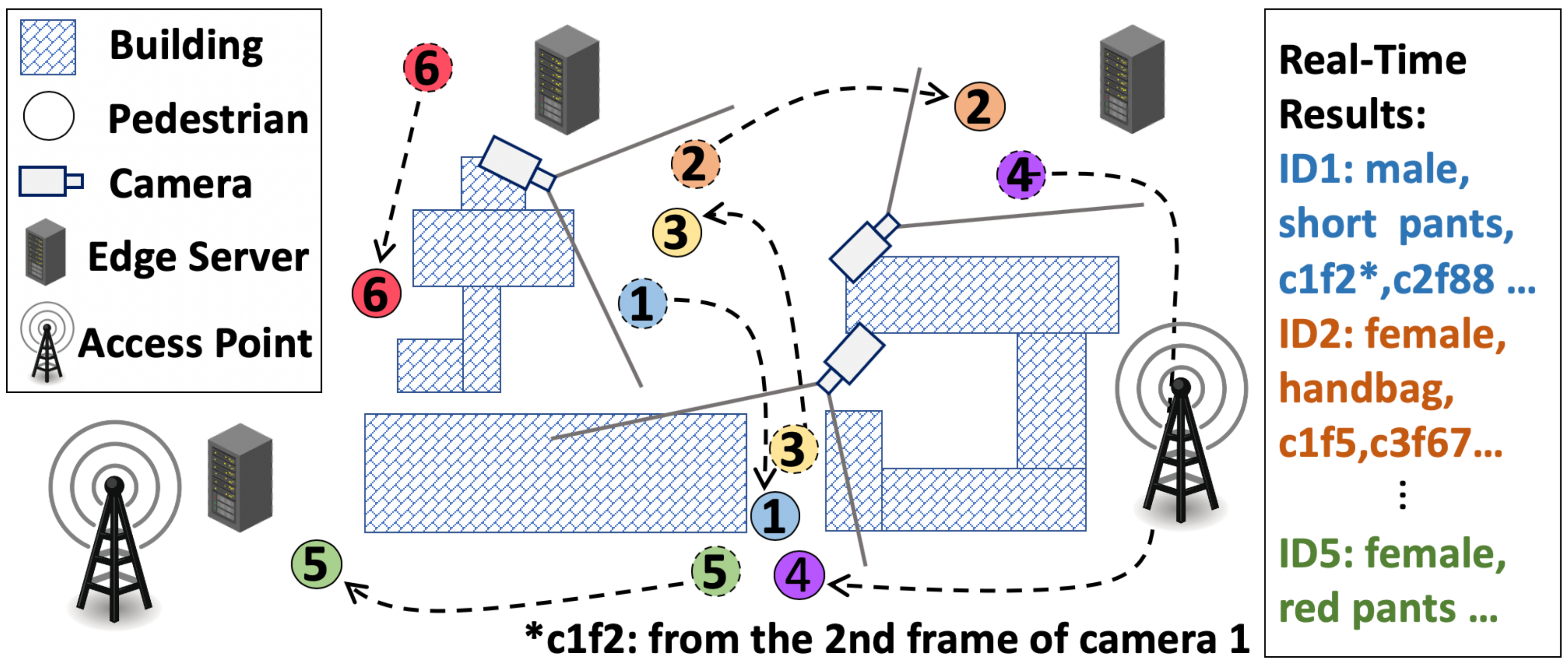}}
\vspace{-3mm}
\caption{An example of the  MEC-enabled camera network for joint attribute recognition and person re-ID.}
\vspace{-6mm}
\label{topology}
\end{figure}

\subsection{Uncertainties of latencies and inference requests}
Processing video streams in an MEC-enabled network incurs latency, which includes the transmission delay by transferring videos or features from cameras to edge servers and the processing delay in edge servers. Such latencies in different edge servers and communication links in the MEC vary significantly from time to time. Notice that we did not consider the transmission delay by transferring analysis results from edge servers to end devices, since the size of analysis results is typical small.  

Let $r_j$ be an inference request for processing in the MEC-enabled network. Each $r_j$ carries an amount of video data for processing. 
The delay experienced by $r_j$ for processing its data stream in an edge server depends on both its data rate and the workload of the server during time slot $t$. It thus varies in different time slots and is usually not known in advance. Since the inference request is scheduled per time slot, we assume that the delay of experienced by each request does not change during time slot $t$, and can be obtained at the very beginning of time slot $t$. 
Similarly, the latency of transmitting the video stream of inference request $r_j$ in each link $e\in E$ varies time by time and is uncertain when $r_j$ arrives into the system. 



\subsection{Problem definition}
Given an MEC-enabled camera network $G = (C\cup V; E)$, a set of online inference requests $R(t)$ at a beginning of each time slot $t$, the {\it distributed inference problem for real-time attribute recognition and person re-ID} is to partition an inference request into different modules and distributing the modules to edge servers, such that the accuracies of person re-ID and attribute recognition are maximized while the latency of inference is minimized, subject to the resource constraints of edge servers. 

\section{A Distributed Inference Framework}\label{sec04}

\subsection{The framework design}
We observe that there is a lot of redundant information in the original video streams collected by cameras in each request $r_j$. 
To avoid the long delay incurred of transmitting such information, we use a pedestrian detector in each camera to extract effective images of pedestrians from the video stream. 
In addition, to recognize the attributes and id of a pedestrian, conventional methods train two separate models for attribute recognition and re-ID, respectively. 
This however may not be efficient in an MEC-enabled network, because two CNN models need to 
be executed and stored in each edge server, thereby introducing additional processing delays and occupying more storage resource. 
Motivated by this, we design a single CNN model for both person re-ID and attribute recognition. 

Our basic idea is to partition the inference framework to different modules, with a minimum amount of inter-module communications. These modules can be distributed to different edge servers, realizing distributed inference. Besides, multiple modules can be run in a server. Since the modules need to be stored in the edge servers, the size of the trained model of the modules is also minimized. 

In pedestrian re-ID, calculating feature similarity or distance between the query image 
and all the pedestrian images in gallery is time consuming, where the gallery is a set of photos held by the system. 
With the continuous running of the system, the gallery can be very large, such that the edge server with the gallery becomes the bottleneck of the system. 
To accelerate this process, we consider the distributed gallery for parallel computing of pedestrian re-ID. Specifically, we split a single centralized gallery into multiple distributed galleries. We store labeled features of pedestrian images in the gallery, instead of storing images in the conventional method. After obtaining the attributes of a person, we need to look-up the whole distributed gallery group to identify the person. 
The final analysis result is shown in Figure~\ref{topology}. The system records an instance for each detected person, which contains the id, the attributes and the frame information. The frame information describes when and where the person was detected.

As shown in Figure~\ref{inference}, the proposed distributed inference framework consists of the following key modules: 
\begin{figure}[t]
\centerline{\includegraphics[scale=0.2]{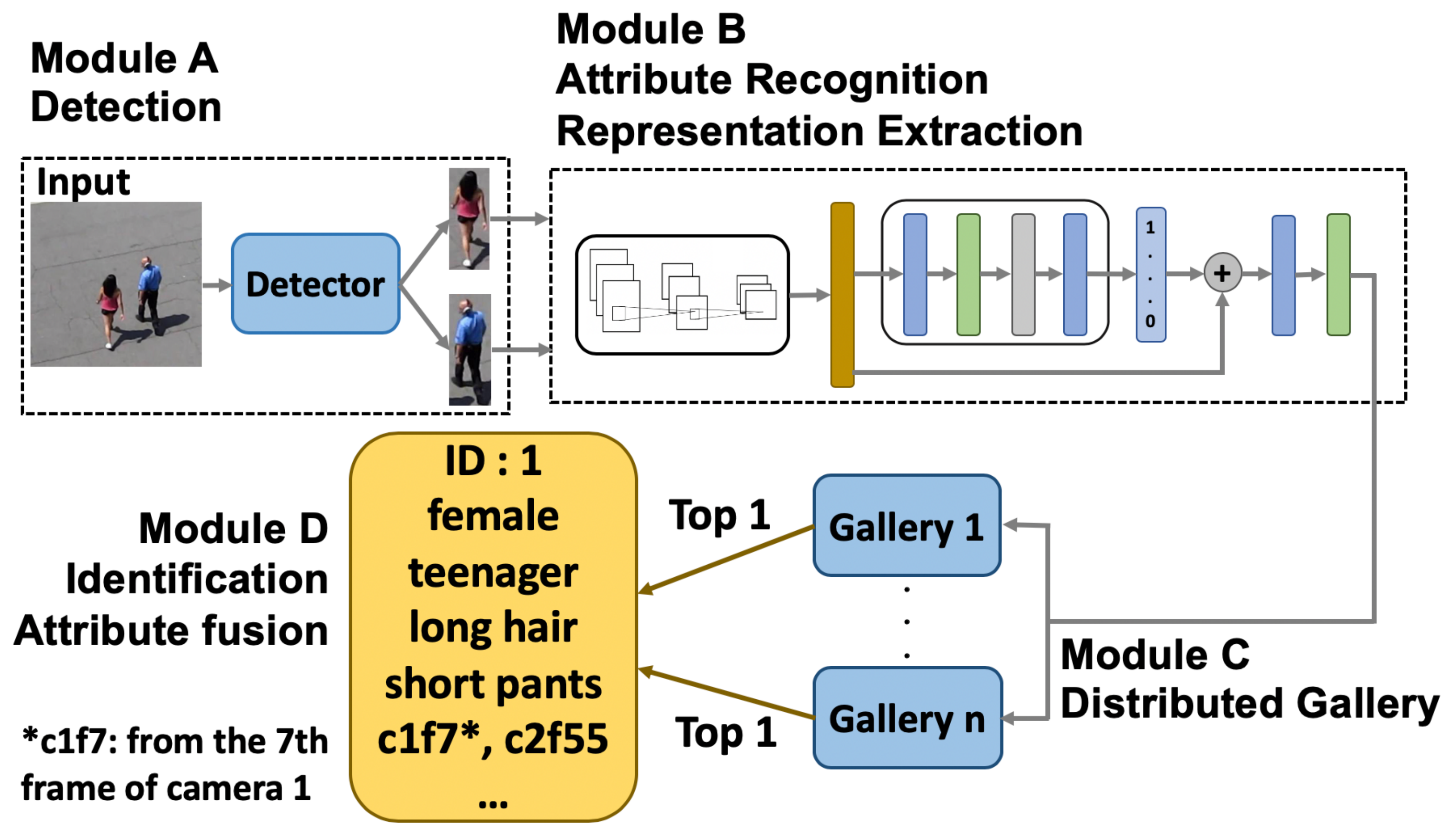}}
\vspace{-3mm}
\caption{The distributed inference framework for joint attributes recognition and person re-ID.}
\vspace{-5mm}
\label{inference}
\end{figure}

{\bf Module A:} This module is called pedestrian detector that runs in each camera $c_i$. The pedestrian detector detects pedestrian images from real-time video stream, and transmits the obtained images as the query persons to the successor modules of the proposed inference framework. 
Note that the frame information is always attached to the inter-module transmission, although we will not explicitly indicate it due to its negligible size. 

{\bf Module B:} This module receives the pedestrian images from the pedestrian detector of each camera $c_i$. It then recognizes the attributes of the person by employing CNN model. Eventually, the output of attribute prediction is concatenated with the extracted feature. The subsequent two layers are the full connection layer and the batch normalization layer. Finally, the output feature will be sent to each instance of Module C for re-ID. 

{\bf Module C:} 
There are multiple instances of Module C in the framework. Each instance is distributed to different edge servers and holds a unique distributed gallery. The $gallery_{i}$ stores the features of pedestrians from $c_{i}$. After receiving the query feature from Module B, all instances of Module C calculates the similarity between the query and all the features stored in parallel. The similarity will be sorted in descending order. Afterwards,  the maximum similarity with its corresponding ID label will be sent to Module D.

{\bf Module D:} This module receives the local maximum labeled similarity from every instance of Module C, and sorts them in descending order. It then obtains the global maximum similarity. We define a threshold $\varepsilon$, which is discussed statistically in Section.~\ref{sec6.3}. If the global maximum similarity is greater than $\varepsilon$, it shows that the person has been detected by the system, and the ID of the query person is the one with the global maximum similarity. 
Then, the new attributes need to be fused with the old attributes in the instance by the method based on exponentially weighted average. 
Otherwise, the query person has not been detected by any camera. In this case, a new instance for the person needs to created, by assigning the person with a unique ID and storing its attributes. 
In either case, the frame information will be added and the id result needs to be sent back for gallery update. 

\subsection{Model training}
Existing joint training models of attribute recognition and re-ID usually consider the recognition of each attribute as a single task, by using separate classifier for each attribute~\cite{lin2019improving}. 
This however contains many redundant information and leads to large training models that cannot be effectively distributed in the MEC-enabled camera network. 
Instead, we consider all the attributes at the same time and the relationship among attributes are learnt simultaneously. Partially similar to~\cite{lin2019improving}, we adopt a multi-task network in the model training, which can learn an identity classifier and an attribute classifier at the same time. Regarded as auxiliary cues, the output of attribute prediction is concatenated with the feature extracted by the CNN. 
The concatenated feature is the input of the identity classifier, as shown in Figure~\ref{training}. 
For re-ID, the loss function $Loss_{reid}$ is a categorial cross entropy over identity label, i.e., 
\begin{equation}
Loss_{reid}=-\frac{1}{N}\sum\nolimits^{N}_{i=1}\sum\nolimits^{K}_{k=0} y_{ik}\log(\hat{p}_{ik})\label{eq1},
\end{equation}
where $\hat{p}_{ik}$ is the predicted probability that $sample_{i}$ is the person $k$. $y_{ik}$ is the ground truth label indicates whether the $sample_{i}$ is the person $k$ or not. $N$ is the number of training samples. $K$ represents the number of identities in the training set.

To minimize the size and improve the accuracy of the trained model, we reduce the redundant and inefficiency in attribute recognition, by considering all the attributes at the same time and learning the relationship among attributes simultaneously~\cite{li2015multi}. 
The loss function $Loss_{attr}$ is a sigmoid cross entropy loss function which learns all the attributes jointly:
\begin{equation}
Loss_{attr}= -\frac{1}{N}\sum\nolimits^{N}_{i=1}\sum\nolimits^{M}_{j=0}w_{j}( y_{ij}\log(\hat{p}_{ij})+(1-y_{ij})\log(1-\hat{p}_{ij})), \nonumber
\end{equation}
with a loss weight for attribute $j$ to deal with the unbalancedness of attribute distributions, i.e., 
\begin{equation}
w_{j}=exp(-\rho _{j}/\sigma ^{2})\label{eq4},
\end{equation}
where $\hat{p}_{ij}$ is the predicted probability for the attribute $j$ of $sample_{i}$. $y_{ij}$ is the ground truth label indicating whether the $sample_{i}$ has the attribute $j$ or not. 
$\rho_{j}$ is the ratio of attribute $j$ in the training set. $\sigma$ is a hyper-parameter. $M$ is the number of attributes.
We define the overall loss function:
\begin{equation}
Loss = \lambda Loss_{reid}+ ({(1-\lambda)}/{M})Loss_{attr}\label{eq5},
\end{equation}
where $\lambda$ is a hyper parameter to balance the two sub losses, which is set to 0.5 in our  experiments. 

\begin{figure}
\centerline{\includegraphics[scale=0.188]{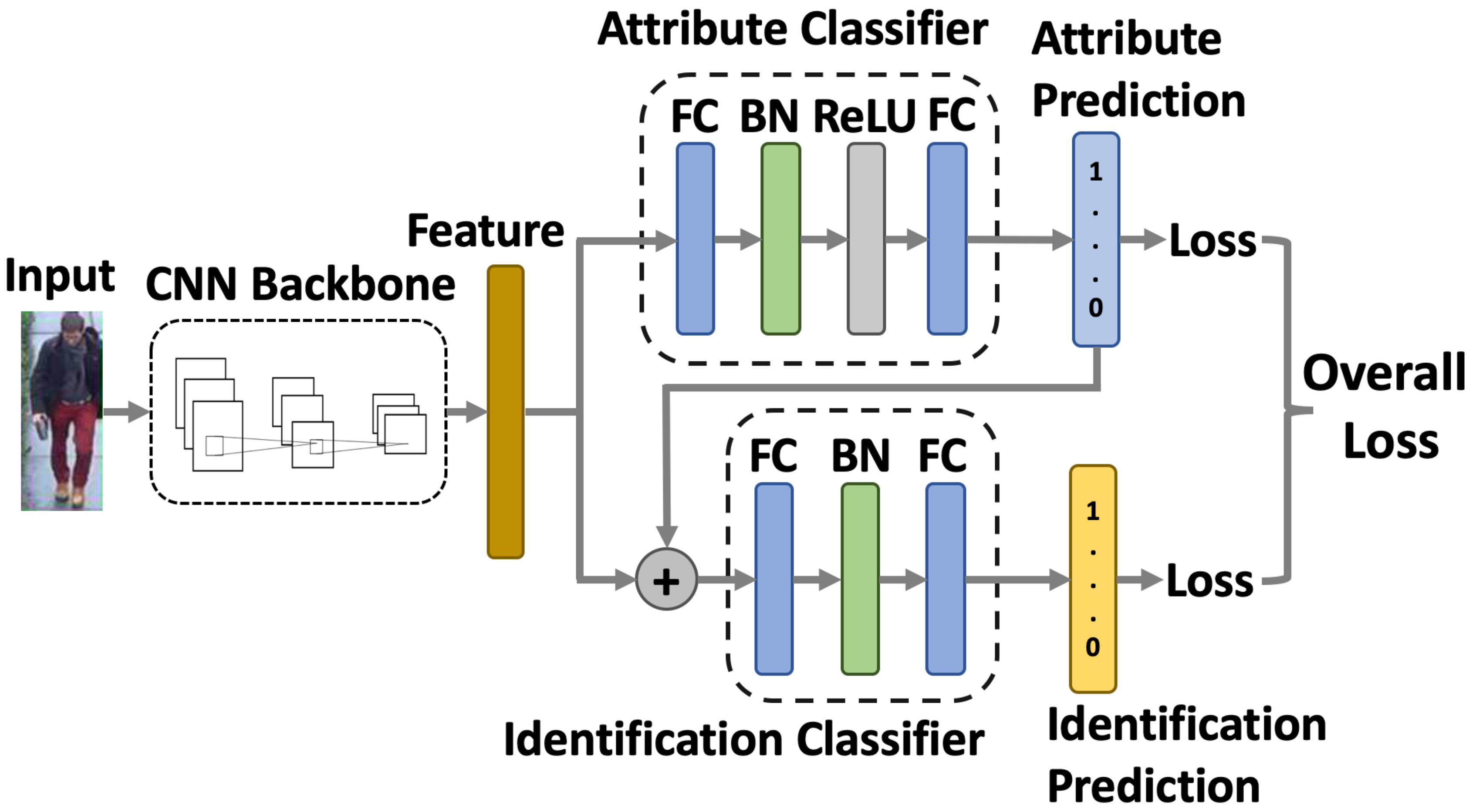}}
\vspace{-3mm}
\caption{Training model}
\vspace{-5mm}
\label{training}
\end{figure}

\section{Module Distribution based on Contextual Multi-Armed Bandits} \label{sec05}

The basic idea of the proposed module distribution algorithm is to adopt an online-learning framework based on the Contextual Multi-Armed Bandit model, as shown in Figure~\ref{ModDistMAB}. In the proposed inference framework, the locations of Modules B and C play a vital role in the latency of online inference requests. We observe that the network latencies, size of gallery and their popularity depend on a {\it context} of the current camera network and the environment it monitors, such as the busyness of a camera, the time (rush hours or holiday hours), and etc. Thus, the proposed online-learning algorithm does not run all the time but re-learns when the context space changes dramatically. Normally, such contexts can be observed by the agents before making decisions. For simplicity, we consider the transmission and processing delays as the contexts of the network. 

\begin{figure}[htbp]
\centerline{\includegraphics[scale=0.28]{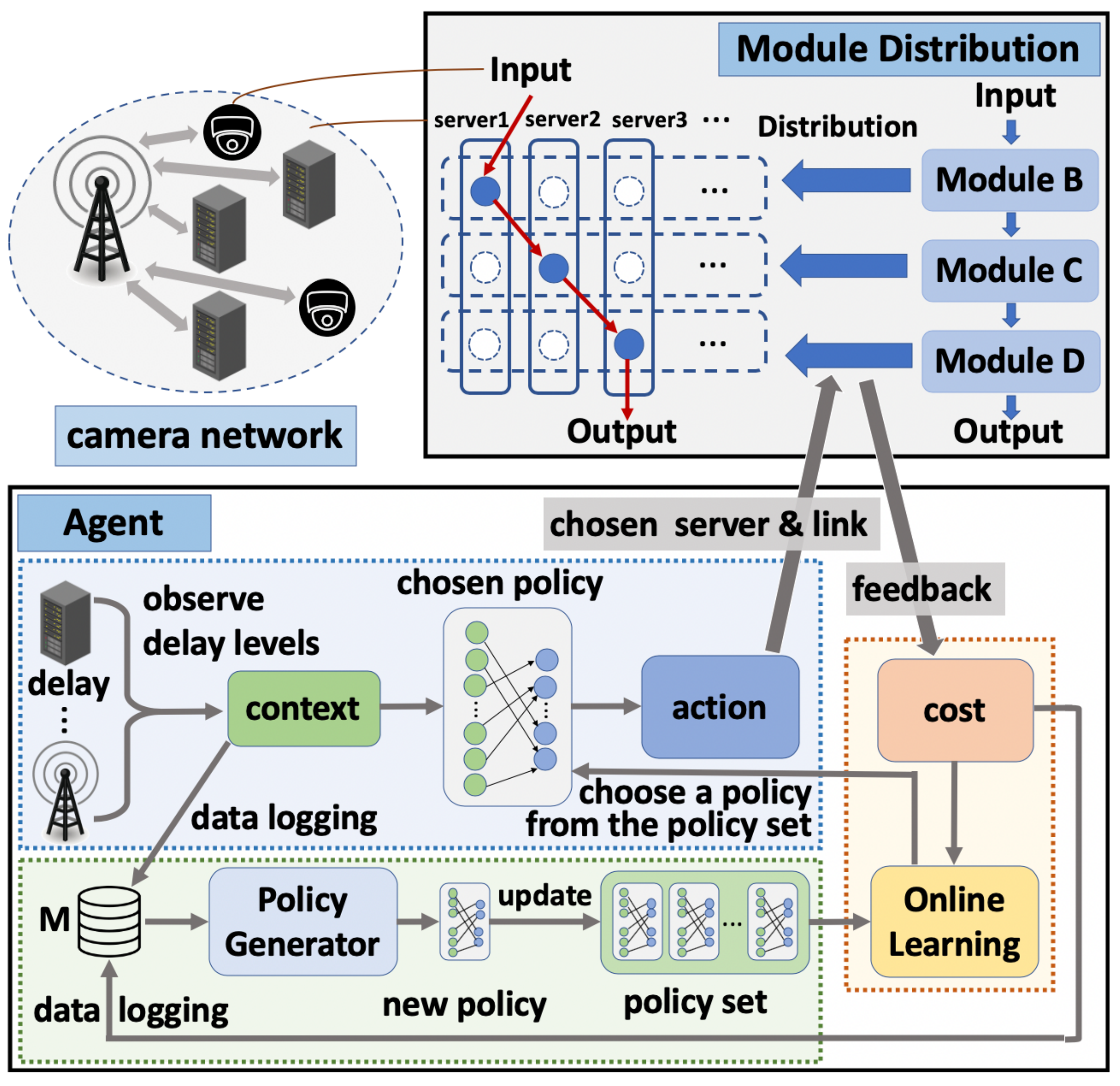}}
\vspace{-3mm}
\caption{The online-learning framework.}
\vspace{-3mm}
\label{ModDistMAB}
\end{figure}

We consider that there is an agent representing each to-be-distributed modules. 
For an edge server $v$ and a link $e$ that connects to $v$, the agent of each module decides whether to assign its module to $v$ and transmit its input data via link $e$.  
An edge server $v$ with a link $e$ that connects to $v$ is considered as an arm, also known as an action. Each of such agents needs to choose an {\it arm} based on the current context of the camera network. For clarity, we define a {\it policy} as a mapping from contexts to arms. Denote by $\Pi$ the set of policies from which an agent chooses its policy.  
We divide the processing and transmission delays into $L$ levels, with each level $l$ ($1 \leq l \leq L$) denoting a fixed range of delay. 
Let $d^{max}_v$ and $d^{min}_v$ be the maximum and minimum values of the delay experienced by request $r_j$ in server $v$ in time slot $t$ of a monitoring period $T$, 
and denote by $d^{max}_e$ and $d^{min}_e$ the maximum and minimum values for of links. 
 The context thus consists of different levels of processing and transmission delays of the camera network. Contexts arrive as independent elements from some fixed contexts set. 
 The arm of each agent thus can be represented as a set of  $\{0, 1\}$, where $1$ denotes the selection of $v$ and $e$ while $0$ denotes they are not selected. 

We use the normalized end-to-end delay as the feedback, denoted by $cost$.
The delay of a certain task is approximately fixed when the computation and transmission conditions are the same. So we assume that the $cost$ with certain arm and context is chosen by a deterministic oblivious adversary before each round. The $cost$ of round t is denoted by $c_{t}(a)$, where the observed context of round $t$ is $x_{t}$ and $a$ is the chosen arm.
The cost of policy $\pi$ after $T$ rounds is:
\begin{equation}
cost(\pi)=\sum\nolimits_{t=1}^{T}c_{t}(\pi(x_{t})).
\end{equation}

In ~\cite{slivkins2019introduction}, an existing benchmark of contextual bandits is the best-response policy: 
\begin{equation}
\pi^{*}(x_{t})=\min\nolimits_{a\in{\mathcal A}}c_{t}(a), 
\end{equation}
where $\mathcal A$ is the set of arms. 
The corresponding regret is: 
\begin{equation}
R^{*}(T)=cost^{OL}-cost^{*}, 
\end{equation}
where the $cost^{OL}$ is the total regret accumulated in the $T$ rounds and $cost^{*} = cost(\pi^{*})$.
The regret is based on the assumption: $\pi^{*}\in\Pi$ for each round t. However, the $\Pi$ in {\bf Algorithm}~~\ref{alg02} {\tt Online Learning} is updated every $I$ rounds. The best-response policy $\pi^{*}$ may not initially be in $\Pi$, but be added through an update. So we now define a more general regret:
\begin{equation}
\label{expectedregret}
R(T)=cost^{OL}-cost^{\sharp}, 
\end{equation}
where $cost^{\sharp}=\sum_{t=1}^{T}\min_{\pi\in\Pi}c_{t}(\pi(x_{t}))$. Note that $R^{*}(T)$ is a special case of $R(T)$ when $\pi^{*}$ is initially in $\Pi$.

The detailed steps of the proposed {\tt ModDistMAB} algorithm are illustrated in {\bf Algorithm}~\ref{alg01}. It contains two sub algorithms {\tt Online Learning} and {\tt Policy Generator}, which are referred to as {\bf Algorithm}~\ref{alg02} and {\bf Algorithm}~\ref{alg03}.

\begin{algorithm}[t]\footnotesize
  \caption{{\tt ModDistMAB}}\label{alg01}
  \begin{algorithmic}[1]
  \REQUIRE ~~ 
  $L$: The number of delay levels;
  ${\mathcal A}$: The set of all arms;
  $M$: The memory that holds the statistical information of the delays of the camera network;
  $N$: Data collection duration;
  $P$: The number of policies;\\
  \STATE Initialize a memory $M$ that holds the statistical information of the delays of the camera network. Initialize an empty policy set $\Pi$; 
  \FOR{each time slot $t\gets 1, \cdots, N$}
  \STATE Observe the processing delay $d_{t}(v)$ for each edge server $v$ and the transmission delay $d_{t}(e)$ for each link $e$;
  \STATE Pick an arm $a_{t}$ from ${\mathcal A}$ uniformly;
  \STATE Observe the cost $c_{t}(a_{t})$ of the chosen arm;
  \STATE Add the data point $(d_{t}(\cdot), a_{t}, c_{t}(a_{t}))$ to the memory $M$;
  \ENDFOR
  \STATE Get the maximum and minimum delays $d^{max}_v$, $d^{min}_v$, $d^{max}_e$ and $d^{min}_e$ from the memory $M$;
  \STATE Generate the context set ${\mathcal X}$ of the camera network with the delay ranges  being divided into L levels;
  \STATE Convert all of the data points $(d_{t}(\cdot), a_{t}, c_{t}(a_{t}))$ to $(x_{t}, a_{t}, c_{t}(a_{t})), x_{t}\in{\mathcal X}$;
  \STATE Call {\bf Algorithm}~~\ref{alg03} Policy Generator with different $\sigma$ and $S$ to get $P$ new policies;
  \STATE Add the $P$ new policies to the policy set $\Pi$; 
  \STATE Online Learning: call {\bf Algorithm}~~\ref{alg02}.
  \end{algorithmic}
\end{algorithm}

\begin{algorithm}[t]\footnotesize
  \caption{{\tt Online Learning}}\label{alg02}
  \begin{algorithmic}[1]
  \REQUIRE ~~
  ${\mathcal X}$: The set of all contexts;
  $\Pi$: The set of all policies;
  ${\mathcal A}$: The set of all arms;
  $M$: The memory that holds the statistical information of the delays of the camera network;
  $I$: The update cycle of policy;
  $\epsilon$: The parameter for weight update. $\epsilon \in (0, \frac{1}{2})$;
  $\gamma$: The probability of picking an arm uniformly at random;\\
  \STATE Initialize the weight as $w_{1}(a)=1$ for each arm a;
  \FOR{each time slot $t\gets 1, \cdots, T$}
  \STATE Observe the context of the camera network $x_{t}$;
  \STATE Randomly generate a float number $n$;
  \IF {$n < \gamma$}
  \STATE Pick an arm $a_{t}$ uniformly at random;
  \STATE Observe the cost $c_{t}(a_{t})$ of the chosen arm;
  \STATE Add data point $(x_{t}, a_{t}, c_{t}(a_{t}))$ to the memory $M$;
  \STATE For each policy $\pi$, $w_{t+1}(\pi)=w_{t}$;
  \ELSE
  \STATE For each policy $\pi$, let $p_{t}(\pi )={w_{t}(\pi )}/({\sum_{{\pi}'\in \Pi}w_{t}({\pi}')});$
  \STATE For each arm $a$, let $q_{t}(a)=\sum\nolimits_{\pi \in \Pi:\pi(x_{t})=a}p_{t}(\pi)$
  \STATE Sample a policy $\pi_{t}$ from distribution $p_{t}(\cdot );$
  \STATE Use the policy $\pi_{t}$ to pick an arm $a_{t}=\pi_{t}(x_{t})$;
  \STATE Observe the cost $c_{t}(a_{t})$ of the chosen arm;
  \STATE Add data point $(x_{t}, a_{t}, c_{t}(a_{t}))$ to the memory $M$;
  \STATE Get the maximum cost $c^{max}$ and the minimum cost $c^{min}$ from $M$;
  \STATE Normalize the current cost:
  $c_{t}(a_{t}) = ({c_{t}(a_{t})-c^{min}})/({c^{max}-c^{min}})$
  \STATE For each arm $a$, define fake costs:
  $\hat{c}_{t}(a)=\begin{cases}
\frac{c_{t}({a_{t})}}{q_{t}(a_{t})} & a =  a_{t} \\ 
 0& otherwise
\end{cases}$;
  \STATE For each policy $\pi$, define fake costs: $\hat{c}_{t}(\pi)=\hat{c}_{t}(\pi(x_{t}))$;
  \STATE For each policy $\pi$, update its weight $w_{t+1}(\pi)=w_{t}(\pi)\cdot(1-\epsilon)^{\hat{c}_{t}({\pi})}$;
  \ENDIF
  \IF {$t \equiv 0 \pmod {I}$}
  \STATE Call {\bf Algorithm}~~\ref{alg03} with different $\sigma$ and $S$ to get $(|\Pi|-1)$ new policies;
  \STATE Keep the policy $\pi\in\Pi$ with the largest $w_{t+1}(\pi)$. Replace the other $(|\Pi|-1)$ old policies with the new generated policies, inheriting the $w$ value.
  \ENDIF
  \ENDFOR

  \end{algorithmic}
\end{algorithm}

\begin{algorithm}[t]\footnotesize
  \caption{{\tt Policy Generator}}\label{alg03}
  \begin{algorithmic}[1]
  \REQUIRE ~~ 
  ${\mathcal A}$: The set of all arms;
  $M$: The memory that holds the statistical information of the delays of the camera network;
  $\sigma$: The parameter in arm weight;
  $S$: The dictionary of training strategy, such as learning rate, epoch; ${\mathcal O}$: The interface of machine learning models for classification task. We used NN in our experiments;\\
  \ENSURE A new policy $\pi$: the trained classification model;
  \STATE Get all data points $(x_{t}, a_{t}, c_{t}(a_{t}))$ from the memory $M$;
  \STATE Let $X'$ be the set of $x_{t}$ in data points;
  \STATE For each context $x \in X'$, find the arm $a$ with the smallest cost $c_{t}(a_{t})$ in data points, named as $a^{x}$
  \STATE For each arm $a \in {\mathcal A}$:

  let $\rho _{a}=\sum_{x\in X':a^{x}=a}\frac{1}{|X'|}$, $w_{a}=exp(\frac{-\rho _{a}}{\sigma ^{2}})$

  \STATE We model the policy as a multi-class classification problem in machine learning. The contexts are the samples and the arms are the predicted classes. The set $(x, a^{x}), x \in X'$ is the training data. Each context $x$ is the input feature and $a^{x}$ is the label. $w$ weights the loss function to balance the distribution of arms;
  \STATE Train the model ${\mathcal O}$ with the strategy $S$.
  \end{algorithmic}

\end{algorithm}

\subsection{Regret analysis}
We now analyze the bounds of the regret of algorithms {\tt ModDistMAB} and {\tt Online Learning} in the following theorems. 

\begin{theorem}\label{th02}
The regret of algorithm {\tt Online Learning} with parameter $\gamma \in  [0,\frac{1}{\sqrt{T}})$ is bounded by $O(\sqrt{T|{\mathcal A}|\ln|\Pi|})$. 
\end{theorem}
\begin{proof}
We first introduce three facts:

For all $x > 0$, there exist $\alpha ,\beta \geq 0$, possibly dependent on $\epsilon$:
\begin{equation}\label{fact1}
fact: (1-\epsilon )^{x}<1-\alpha x+\beta x^2.
\end{equation}

For $\epsilon \in(0, {1}/{2})$:
\begin{equation}\label{fact2}
fact: \epsilon <\ln({1}/({1-\epsilon }))<3\epsilon.
\end{equation}

For any $x \in (0,1)$:
\begin{equation}\label{fact3}
fact: \ln(1-x)<-x. 
\end{equation}

We now show the upper bound of $\hat{c}_{t}(\pi)$ and $\hat{c}_{t}(a)$. Let $a=\pi(x_{t})$ be the arm chosen by policy $\pi$. After the normalization, $c_{t}(a)\in \left [ 0,1\right ]$ for each arm $a$. We obtain:
\begin{equation}
\hat{c}_{t}(\pi)=\hat{c}_{t}(a)\leq {c_{t}(a)}/{q_{t}(a)}\leq {1}/{q_{t}(a)}.
\end{equation}

In order to demonstrate that the fake cost $\hat{c}_{t}(a)$ is the unbiased estimate of the true cost $c_{t}(a)$. For each arm $a$, we have:
\begin{equation}
\mathbb{E}[\hat{c}_{t}(a)]=Pr[a=a_{t}]\cdot {c_{t}(a)}/{{q_{t}(a)}}+Pr[a \neq a_{t}]\cdot 0=c_{t}(a).
\end{equation}

We start with the case of $\gamma = 0$, which means that the algorithm always picks an arm by policy. Let $W_{t}=\sum_{\pi\in\Pi}w_{t}(\pi)$, we have:
\begin{equation}
W_{T+1}> w_{T+1}(\pi^{*})=(1-\epsilon )^{cost^{*}}\geq (1-\epsilon )^{cost^{\sharp}}.
\end{equation}

Let $\alpha =\ln (\frac{1}{1- \epsilon}), \beta = \alpha^{2}$ and use the fact in Eq.~(\ref{fact1}), we can get 
\begin{align}
& {W_{t+1}}/{W_{t}} = \sum\nolimits_{\pi\in\Pi}(1-\epsilon )^{\hat{c}_{t}(\pi)}\cdot ({w_{t}(\pi)})/{W_{t}}
\nonumber \\
& \leq \sum\nolimits_{\pi\in\Pi}(1-\alpha \hat{c}_{t}(\pi)+\beta \hat{c}_{t}(\pi)^{2})\cdot p_{t}(\pi)\nonumber\\
&= \sum\nolimits_{\pi\in\Pi}p_{t}(\pi) - \alpha \sum\nolimits_{\pi\in\Pi}p_{t}(\pi)\hat{c}_{t}(\pi)+\beta \sum \nolimits_{\pi\in\Pi} p_{t}(\pi) \hat{c}_{t}(\pi)^{2})\nonumber\\
&=1-\alpha \mathbb{E}[\hat{c}_{t}(\pi)]+\beta \mathbb{E}[\hat{c}_{t}(\pi)^{2}]. \nonumber
\end{align}


To get the upper bound of $\mathbb{E}[\hat{c}_{t}(\pi)^{2}]$, we have:
\begin{equation}
\begin{aligned}
& \mathbb{E}[\hat{c}_{t}(\pi)^{2}]=\sum\limits_{\pi\in\Pi}p_{t}(\pi)\cdot\hat{c}_{t}(\pi)^{2}
=\sum\limits_{a\in{\mathcal A}}\sum\limits_{\pi \in \Pi:\pi(x_{t})=a}p_{t}(\pi)\hat{c}_{t}(\pi)\hat{c}_{t}(\pi)\\
&\leq \sum\limits_{a\in{\mathcal A}}\sum\limits_{\pi \in \Pi:\pi(x_{t})=a}\frac{p_{t}(\pi)}{q_{t}(a)}\hat{c}_{t}(a)
= \sum\limits_{a\in{\mathcal A}}\frac{\hat{c}_{t}(a)}{q_{t}(a)}\sum\limits_{\pi \in \Pi:\pi(x_{t})=a}p_{t}(\pi)\\
&=\sum\nolimits_{a\in{\mathcal A}}\hat{c}_{t}(a).
\end{aligned}
\end{equation}

The mathematical expectation of $\mathbb{E}[\hat{c}_{t}(\pi)^{2}]$ thus is:
\begin{equation}\label{eq18}
\begin{aligned}
& \mathbb{E}(\sum\nolimits_{\pi\in\Pi}p_{t}(\pi)\cdot \hat{c}_{t}(\pi)^{2})
\leq \mathbb{E}(\sum\nolimits_{a\in{\mathcal A}}\hat{c}_{t}(a))\\
&=\sum\nolimits_{a\in{\mathcal A}}\mathbb{E}(\hat{c}_{t}(a))
=\sum\nolimits_{a\in{\mathcal A}}\mathbb{E}(c_{t}(a)) 
\leq |{\mathcal A}|.
\end{aligned}
\end{equation}

According to the fact in Eq.~(\ref{fact3}) and $\alpha \mathbb{E}[\hat{c}_{t}(\pi)] -\beta \mathbb{E}[\hat{c}_{t}(\pi)^{2}] \in \left ( 0,1\right )$, we can derive 
\begin{equation}
\begin{aligned}
&\ln ({W_{t+1}}/{W_{t}})< \ln(1-\alpha \mathbb{E}[\hat{c}_{t}(\pi)]+\beta \mathbb{E}[\hat{c}_{t}(\pi)^{2}])\\
&< -\alpha \mathbb{E}[\hat{c}_{t}(\pi)]+\beta \mathbb{E}[\hat{c}_{t}(\pi)^{2}].
\end{aligned}
\end{equation}

To calculate the expected regret, we need to connect $cost^{OL}$ and $cost^{\sharp}$. By transposing and summing over $t$ on both sides, we have:
\begin{equation}
\begin{aligned}
&\sum _{t\in [T]}(\alpha \mathbb{E}[\hat{c}_{t}(\pi)]-\beta \mathbb{E}[\hat{c}_{t}(\pi)^{2}])=\alpha \mathbb{E}[cost^{OL}]-\beta\sum _{t\in [T]} \mathbb{E}[\hat{c}_{t}(\pi)^{2}]\\
&<- \sum _{t\in [T]}\ln({W_{t+1}}/{W_{t}})
=-\ln\prod_{t\in [T]}{W_{t+1}}/{W_{t}}
=-\ln ({W_{T+1}}/{W_{1}})\\
&=\ln W_{1}-\ln W_{T+1}
<\ln|\Pi|-cost^{\sharp}\ln(1-\epsilon)\\
&=\ln|\Pi|+\alpha \mathbb{E}[cost^{\sharp}]. 
\end{aligned}
\end{equation}

We now calculate the expected regret:
\begin{equation}
\begin{aligned}
& \mathbb{E}[cost^{OL}-cost^{\sharp}]< ({\ln|\Pi|})/{\alpha}+\alpha \sum \nolimits_{t\in[T]}\mathbb{E}[\hat{c}_{t}(\pi)^{2}]\\
&\leq ({\ln|\Pi|})/{\alpha}+\alpha T|{\mathcal A}|
\leq ({\ln|\Pi|})/{\epsilon }+3 \epsilon T|{\mathcal A}|,
\end{aligned}
\end{equation}
where the upper bound of $\mathbb{E}[\hat{c}_{t}(\pi)^{2}]$ is obtained in Eq.~(\ref{eq18}) and the last step is based on the fact in in Eq.~(\ref{fact2}). 
Let $\epsilon =\sqrt{({\ln|\Pi|})/({3T|{\mathcal A}|})}$, and we have
\begin{equation}
\begin{aligned}
\mathbb{E}[R(T)]=\mathbb{E}[cost^{OL}-cost^{\sharp}]< 2\sqrt{3}\sqrt{T|{\mathcal A}|\ln|\Pi|}, 
\end{aligned}
\end{equation}
due to Eq.~(\ref{expectedregret}). 

For the extended case of $\gamma \in  [0,\frac{1}{\sqrt{T}})$, each round has the probability $\gamma$ to pick an arm randomly, contributing at most 1 to the expected regret. So we have:
\begin{align}
& \mathbb{E}[R(T)]< 2\sqrt{3}\sqrt{T(1-\gamma)|{\mathcal A}|\ln|\Pi|}+\gamma T\nonumber\\
&\leq 2\sqrt{3}\sqrt{T|{\mathcal A}|\ln|\Pi|}+\sqrt{T}
\leq O(\sqrt{T|{\mathcal A}|\ln|\Pi|}).\nonumber
\end{align}
\end{proof}

\begin{theorem}\label{th01}
The regret of algorithm {\tt ModDistMAB} is bounded by $O(\sqrt{T |E||V|})$. 
\end{theorem}
\begin{proof}

To obtain the bound of expected regret, we show the size of arm space and policy space. 
First, we know that there are $|V|$ servers in the MEC-enabled camera network and $|E|$ links. Denote by $K$ the number of arms in the arm space. Clearly, we have $K = O(|V||E|)$. 
Second, for the policy space, according to algorithm {\tt ModDistMAB}, the processing and transmission delays are divided into $L$ levels. There are at most $O(L^{|E||V|})$ contexts for all edge servers and links. By enumeration, there are $O(K^{L^{|E||V|}})$ possible policies in $\Pi$. As shown in Theorem~\ref{th02}, we know that the regret bound of {\bf Algorithm}~~\ref{alg02} is $O(\sqrt{T|{\mathcal A}|\ln|\Pi|})$, where $|{\mathcal A}|$ is the size of arm set. Plugging in this $\Pi$, we obtain the regret is $O(\sqrt{T|E||V|L^{|E||V|}\ln(|E||V|)})$. Due to the exponential dependence on $|E||V|$, the regret increases dramatically with the scaling up of the system. In addition, the running time per round of {\bf Algorithm}~~\ref{alg02} is $O(K^{L^{|E||V|}})$, which reflects the algorithm is extremely slow in practice. To solve this problem, the key is to reduce the size of policy set $\Pi$. Instead of getting all possible policies by enumeration, we use {\bf Algorithm}~~\ref{alg03} to generate several suboptimal policies. Obtaining the suboptimal policy set, {\bf Algorithm}~~\ref{alg02} finds the optimal one by online learning. Denote by $P$ the number of generated policies. We obtain the regret bound of {\bf Algorithm}~~\ref{alg02} is $O(\sqrt{T|E||V|\ln P})$. Both regret bound and running time are significantly reduced.

In {\tt ModDistMAB}, the algorithm picks the arm uniformly to collect data for the first $N$ rounds. Each round in this duration contributes at most $1$ to regret. The {\bf Algorithm}~~\ref{alg02} runs in the subsequent $(T-N)$ rounds. So we can get the regret bound
\begin{align}
O(\sqrt{(T-N)|E||V|\log P} + N)= O(\sqrt{T|E||V|}),\nonumber
\end{align}
assuming that $N$ and $P$ are the given constants. 
\end{proof}

\section{Experiments}\label{sec06}

\subsection{Training and test-bed settings}
For the training process, we adopt ResNet-50~\cite{he2016deep} as the backbone, which is pre-trained on ImageNet~\cite{russakovsky2015imagenet}. For pedestrian re-ID, we append a 512-dim fully connected layer, a batch normalization layer and a dropout layer after the pool5 layer~\cite{lin2019improving}, without applying the ReLU activation function. 
For the attribute recognition, we use a similar structure as re-id part. The difference is that we append a ReLU layer after batch normalization layer. 
For the training strategy, the number of epochs is 80. The batch size is 32. The learning rate is 0.02 and warm up in 10 epochs. The input image will be resized to 384$\times$128 and random erasing is applied. 

In our test-bed, we use two Raspberry Pi 4B boards (4GB RAM) with cameras, as shown in Figure~\ref{testbed}~(a). We adopt the OpenVINO toolkit and use a pre-trained pedestrian detector in it. Then we use Intel Neural Compute Stick 2 plugged in USB ports to accelerate the inference of CNN. Raspberry Pi will send the detected pedestrian images to the edge nodes for further inference. The test-bed also consists of four edge servers with GPU (2080 Ti) that are interconnected by five hardware switches, as shown in Figure~\ref{testbed}(b). 

\begin{figure}[htbp]
\centerline{\includegraphics[scale=0.08]{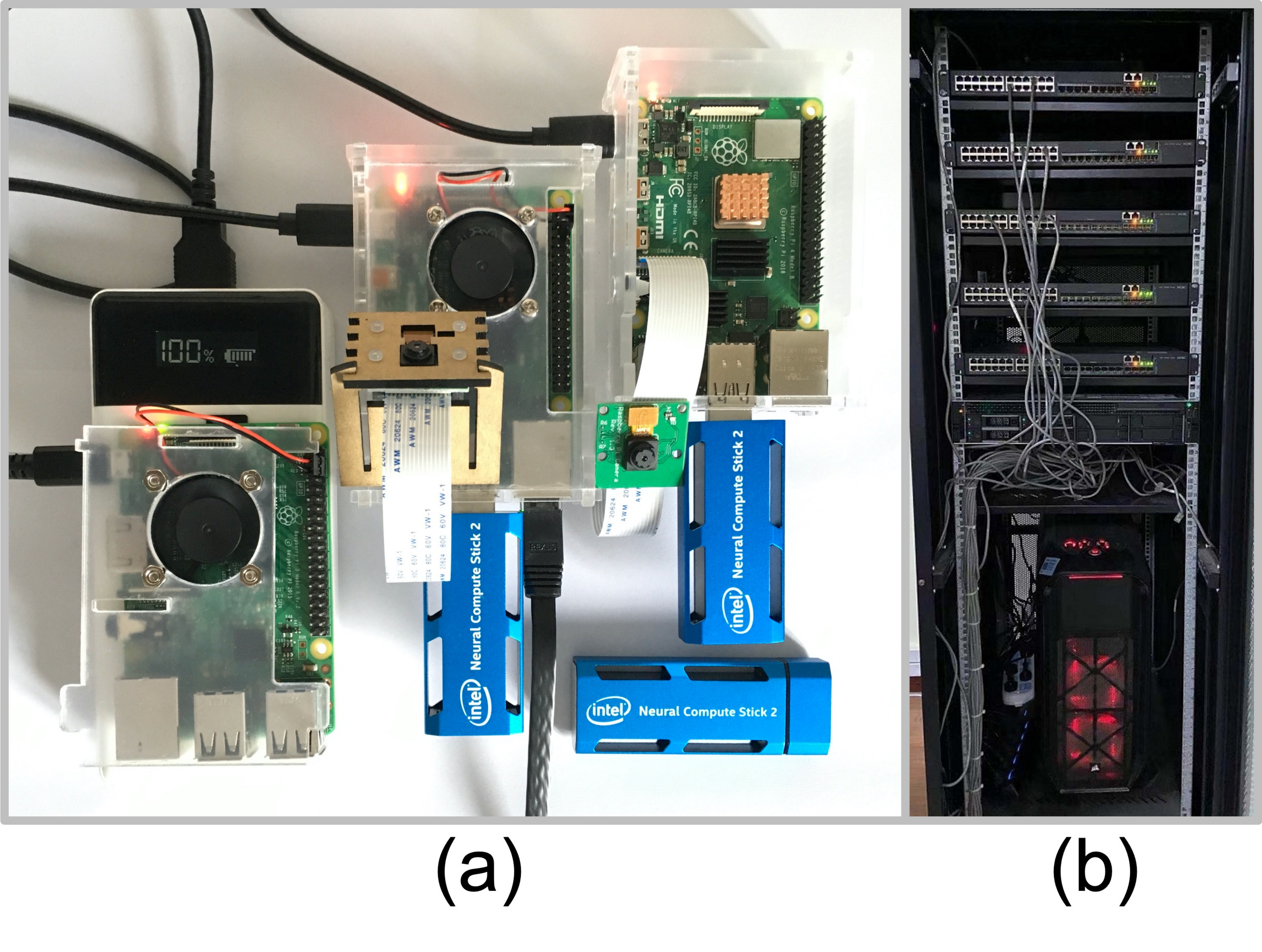}}
\vspace{-5mm}
\caption{A test-bed for the MEC-enabled camera network.}
\vspace{-6mm}
\label{testbed}
\end{figure}

\subsection{Datasets and baseline algorithm}
The accuracy is evaluated on two pedestrian re-ID datasets with attribute annotation: Market-1501~\cite{zheng2015scalable} and DukeMTMC-reID~\cite{zheng2017unlabeled}. 
We use video sequences from MOTChallenge~\cite{milan2016mot16}~\cite{leal2015motchallenge} and evaluate our distributed algorithm for acceleration on them, which are filmed with static and elevated cameras in unconstrained environments. 
For the sake of fairness, we should compare our method with the similar one that considering the joint training of attribute recognition and person re-ID.
So we choose an existing study by~\cite{lin2019improving}, which is referred to as Baseline. 
\begin{table}[h]
 \resizebox{0.86\textwidth}{!}{
 \centering
 \begin{minipage}{\textwidth}
\begin{tabular}{lllll}
\hline
Methods  & Rank-1 & Rank-5 & Rank-10 & mAP  \\
\hline
Baseline        & 87.04 & 95.10 & 96.42 & 66.89  \\
Ours            & 93.05 & 97.57 & 98.34 & 81.60  \\
$Ours_{framejunk}$  & {\bfseries 96.44} & {\bfseries98.72} & {\bfseries99.23} & {\bfseries84.30}  \\
\hline
\end{tabular}
\end{minipage}}
\vspace{0mm}
\caption{CMC and mAP on Market-1501}
\vspace{-10mm}
\label{reid_market}
\end{table}

\begin{table}[h]
\centering
 \resizebox{0.86\textwidth}{!}{\begin{minipage}{\textwidth}
\begin{tabular}{lllll}
\hline
Methods  & Rank-1 & Rank-5 & Rank-10 & mAP  \\
\hline
Baseline        & 73.92 & - & - & 55.56  \\
Ours            & 85.37 & 92.55 & 94.84 & 70.75  \\
$Ours_{framejunk}$  & {\bfseries96.59} & {\bfseries98.79} & {\bfseries99.10} & {\bfseries79.83}  \\
\hline
\end{tabular}
\end{minipage}}
\vspace{0mm}
\caption{CMC and mAP on DukeMTMC-reID}
\vspace{-10mm}
\label{reid_duke}
\end{table}

\subsection{Model accuracy results}\label{sec6.3}
The Cumulative Matching Characteristic (CMC) curve and the mean average precision (mAP) are adopted for pedestrian re-ID evaluation. In the previous CMC calculation, if the gallery image is under the same camera as the query image, it is considered as junk image. However, in our application, images from the same camera are used as effective information to identify the current query image. So we only consider the gallery images which are from the same frame as the query image as junk images. Our proposed CNN model based on this new evaluation metric is referred to as $Ours_{framejunk}$, while the proposed model based on conventional CMC is denoted by {\it Ours}.
The results on Market-1501 and DukeMTMC-reID are shown in Tables~\ref{reid_market} and~\ref{reid_duke}, respectively. It can be seen from the tables that $Ours$ has the improvement of 6.0\% in Rank-1 and 14.7\% in mAP compared with the Baseline approach. The reason is that higher attribute accuracy promotes the accuracy of re-ID. Also, we adopt a novel training strategy that is suitable for person re-ID. $Ours_{framejunk}$ has an accuracy of up to 96.44\% for Rank-1, indicating the identification accuracy of our system is up to 96.44\%. The accuracy improvement of $Ours_{framejunk}$ compared with $Ours$ confirms that the images from the same camera are the conducive information for identification. 

For the attribute recognition, we test the classification accuracy for each attribute. For geographically distributed cameras fusion, we fuse the prediction output of the same person in different images, which is referred to as $Ours_{fused}$. The results on Market-1501 are shown in the table~\ref{attribute_market}. $Ours$ is much better than Baseline. The reason is that we adopt a single classifier for all the attributes, which is better than a classifier for each attribute. This allows the relationship among attributes to be learnt simultaneously. In addition, the accuracy improvement of $Ours_{fused}$ compared with $Ours$ means that cross-camera can improve the performance of attribute recognition.

\begin{figure}[htbp]
\centerline{\includegraphics[scale=0.22]{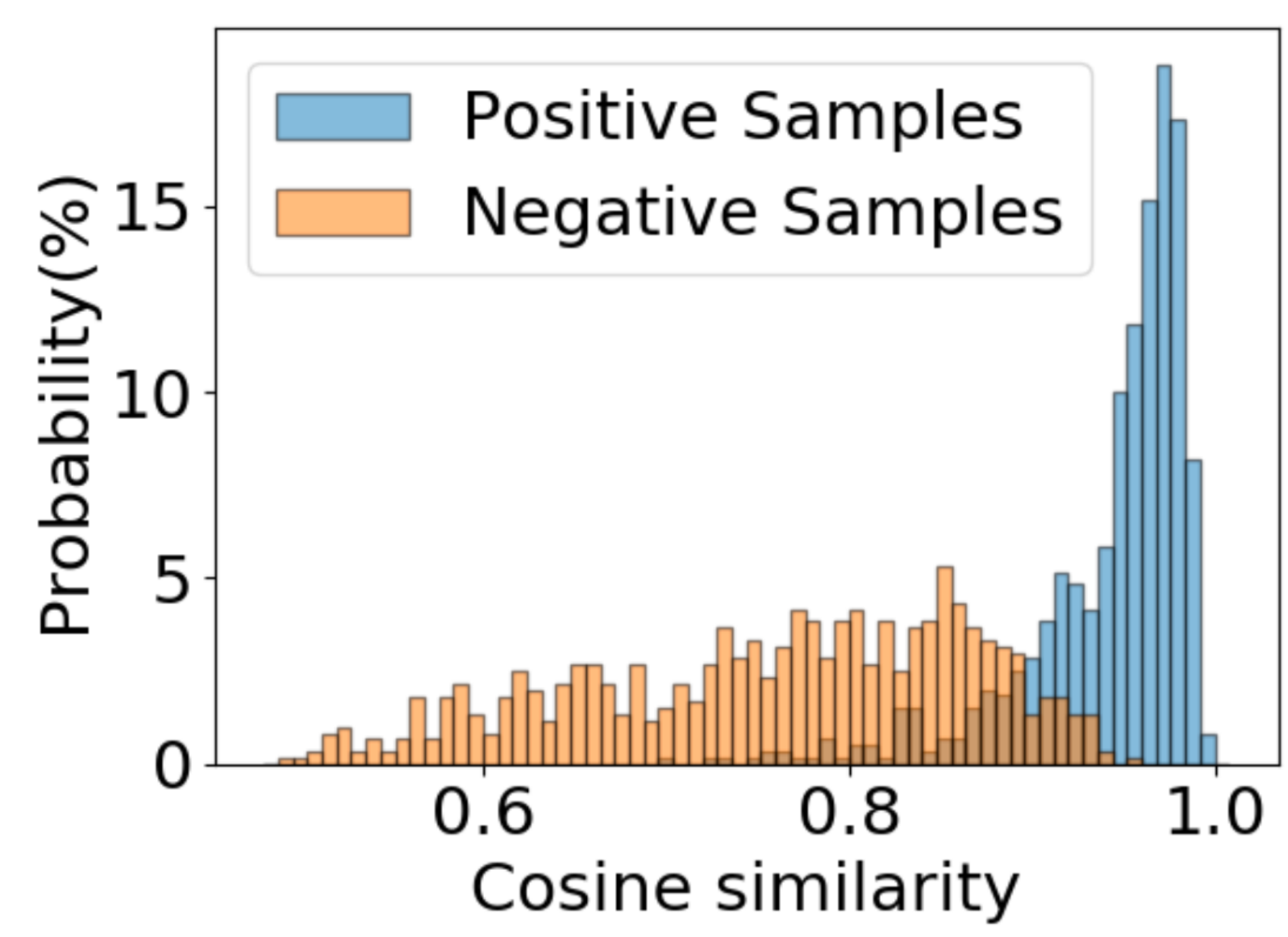}}
\vspace{-3mm}
\caption{The maximum similarity between query images and the gallery}
\vspace{-4mm}
\label{threshold}
\end{figure}

To distinguish whether the current query person exists in the gallery, we make a statistic on the dataset, inspired by~\cite{DBLP:journals/corr/abs-1711-05535}. For each identity in the gallery, we select one image as a positive query sample. For each identity not in the gallery, we select one image as a negative query sample. 
The similarity results are shown in Figure~\ref{threshold}. 
From the results, we can see that when the threshold as $0.9$ the negative and positive samples can be separated sufficiently. We thus set the threshold $\varepsilon$ to $0.9$ in Module D. 

 
\begin{table}
\centering
\resizebox{0.86\textwidth}{!}{\begin{minipage}{\textwidth}
\begin{tabular}{llllllll}
\hline
          &gender&age&hair&L.slv&L.low&S.clth\\
\hline
Baseline        &88.9 &88.6 &84.4 &93.6 &93.7 &92.8 \\
Ours            &93.1 &93.8 &89.3 &93.5 &94.1 &92.6\\
$Ours_{fused}$  &{\bfseries94.5} &{\bfseries95.2} &{\bfseries90.3} &{\bfseries94.0} &{\bfseries95.5} &{\bfseries93.5}\\
\hline
        &B.pack&H.bag&bag&hat&C.up&C.low&Avg\\
\hline
Baseline        &84.9 &90.4 &76.4 &97.1 &74.0 &73.8 &86.6\\
Ours            &87.7 &89.8 &78.5 &{\bfseries97.7} &95.4 &94.5 &91.7\\
$Ours_{fused}$  &{\bfseries91.3} &{\bfseries89.9} &{\bfseries81.7} &97.3 &{\bfseries96.0} &{\bfseries95.0} &{\bfseries92.9}\\
\hline
\end{tabular}
\end{minipage}}
\caption{Attribute recognition accuracy on Market1501-attribute}
\vspace{-10mm}
\label{attribute_market}
\end{table}

\subsection{Testbed experiments}
To evaluate the effectiveness, we compare the proposed {\tt ModDistMAB} algorithm with the other four, as shown in the Figure ~\ref{inference_delay}. In order for the algorithms to converge, we ensure that no other tasks except for our modules are being processed in each edge server and no other data is being transferred in each link. It means that the network environment is static and our modules are the only variables. The {\tt Fixed} algorithm permanently assigns all the modules to the same edge server with the strongest computing power. We can see that this method performs worst. The reason is that the chosen edge server is overloaded, causing prohibitive huge processing overhead. The {\tt Greedy} algorithm always selects the edge server that can achieve the minimum processing latency, based on the observations of the current network context and the historical information. However, the {\tt Greedy} algorithm do not consider the transmission delay, so it converges to a bad delay. {\tt ModDistMAB without policy updating} always uses the initially generated policy set $\Pi$. It can be seen from the Figure~\ref{inference_delay}, its delay is hardly reduced and converges to a large value. We found that the context space is small because of the static network. During the data collection phase, every possible context is recorded. It means that the policy training set completely covers the feature space, so the classification in {\bf Algorithm}~\ref{alg03} does not predict, but records the optimal action before. Although different training strategies are adopted, the initial $P$ policies still have the same mapping relationship. Therefore, the online learning algorithm does not work, and the delay does not decrease with running. Besides, due to insufficient data collection, there are many wrong labels in {\bf Algorithm}~\ref{alg03}. The delay is large because the policies make many mistakes. {\tt ModDistMAB without online learning} only trains one policy based on the constantly updated data and uses it. Because online learning does not work in the static network, its performance is similar to {\tt ModDistMAB}, converging to a small delay.

However, in practice, the network is dynamic and must have other tasks. The context space is large in the dynamic network. So the policies generated by machine learning will vary greatly, which means the online learning algorithm works. We randomly add other computation and transmission tasks to the Testbed to simulate the dynamic network. As shown in Figure ~\ref{dynamic_delay}, we can see that {\tt ModDistMAB} is always better than {\tt ModDistMAB without online learning} after going through 1900 rounds. The reason is that {\tt ModDistMAB without online learning} only uses a suboptimal policy generated by {\bf Algorithm}~\ref{alg03}, while {\tt ModDistMAB} algorithm will further find out  the optimal policy from the set of suboptimal policies through online learning. It shows that the {\tt ModDistMAB} algorithm is the most efficient one in practice.

\begin{figure}[htbp]
\vspace{-3mm}
\centerline{\includegraphics[scale=0.23]{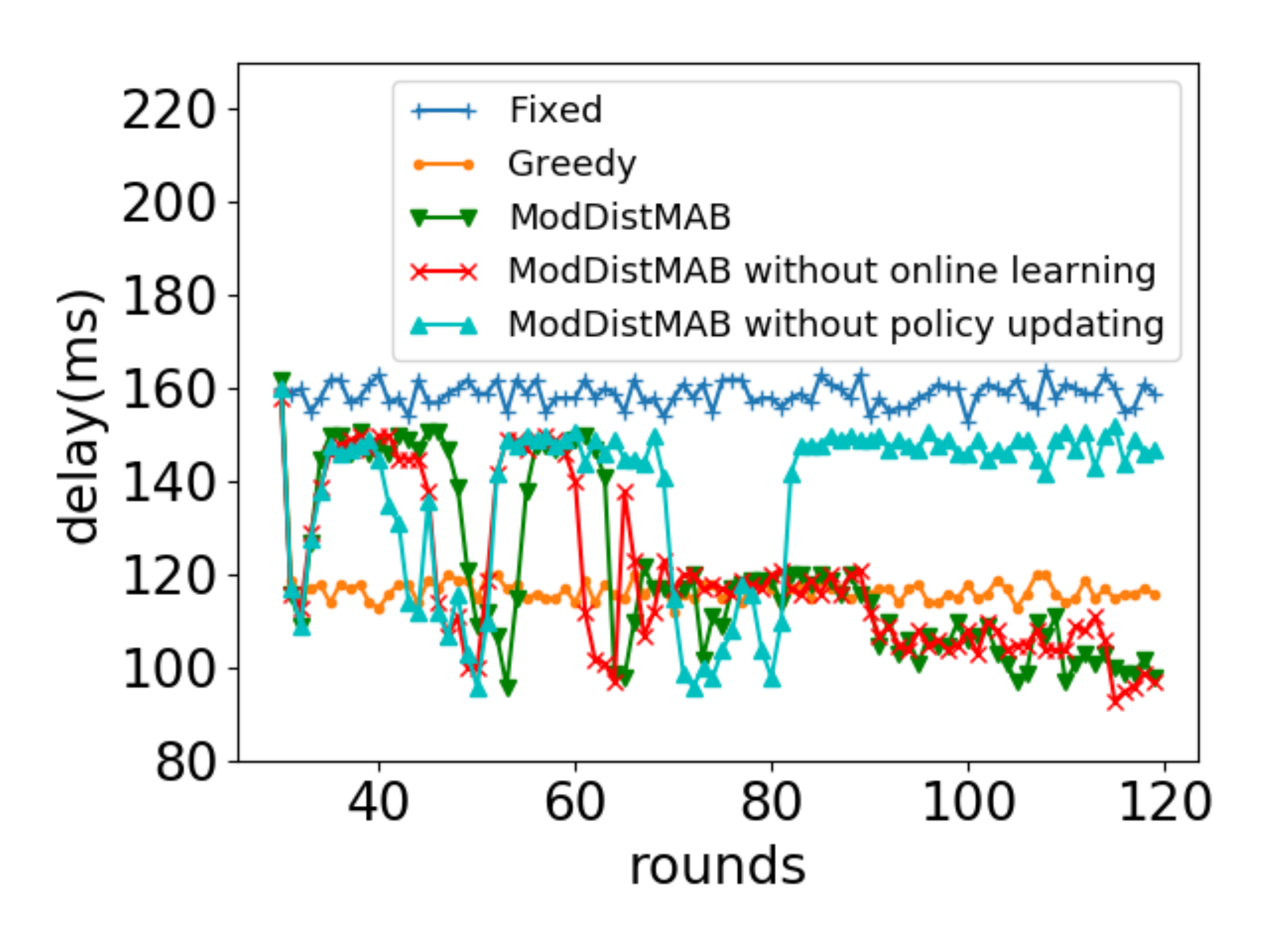}}
\vspace{-5mm}
\caption{The performance of algorithms 
in the static network with a data collection duration of 30. }
\vspace{-8mm}
\label{inference_delay}
\end{figure}

\begin{figure}[htbp]
\centerline{\includegraphics[scale=0.23]{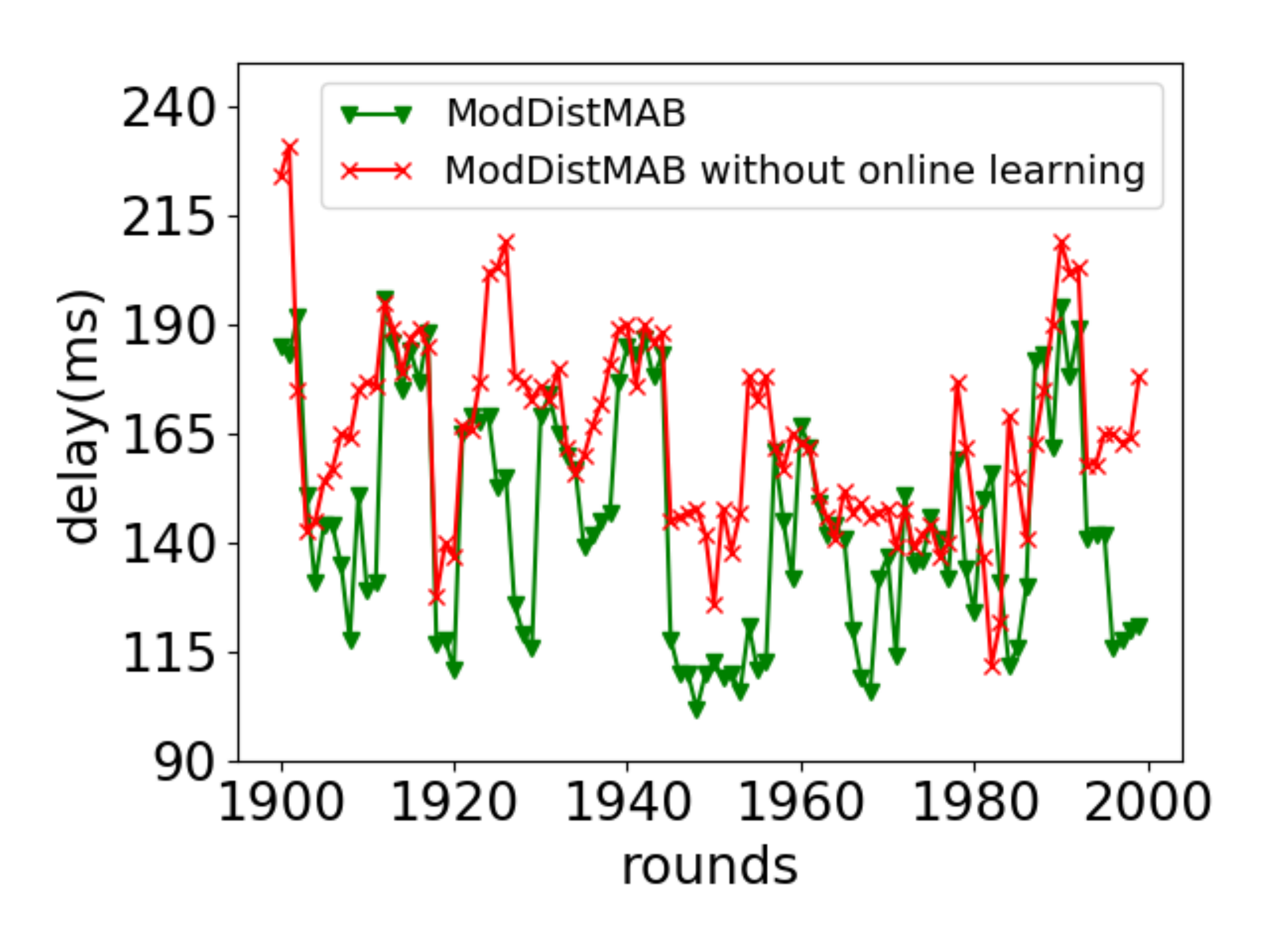}}
\vspace{-5mm}
\caption{The performance of algorithms {\tt ModDistMAB} and {\tt ModDistMAB} without online learning in the dynamic network. In practice, 1900 rounds may take only 6 minutes and the policy may be trained in earlier rounds. }
\vspace{-4mm}
\label{dynamic_delay}
\end{figure}

\section{Conclusions}\label{sec07}
In this paper, we studied the problem of real-time pedestrian attribute recognition and re-ID in an MEC-enabled camera network. We first proposed a novel distributed inference framework with a set of distributed modules jointly considering the attribute recognition and person re-ID. We also devised an algorithm for the module distributions by proposing a novel Contextual Multi-Armed Bandit model, to address the network uncertainties of an MEC-enabled camera network. We then evaluated the performance of the proposed distributed inference framework and algorithm by both simulations with real datasets and system implementation. Results show that the accuracy of the distributed inference framework is up to 96.6\% for identification and 92.9\% for attribute recognition, and the inference delay is at least 40.6\% lower than algorithm {\tt Fixed}. 

\bibliographystyle{ACM-Reference-Format}
\bibliography{PR-MEC}

\end{document}